\documentclass[sigconf]{acmart}

\usepackage{array}
\usepackage{booktabs}
\usepackage{balance}
\usepackage{amsmath}
\usepackage{amsthm}
\usepackage{thmtools}
\usepackage{algorithm}
\usepackage{algorithmic}
\usepackage{newfloat}
\usepackage{listings}
\usepackage{multirow}
\usepackage{xcolor} 
\usepackage{booktabs}
\usepackage{tabularx}

\DeclareMathOperator*{\argmin}{arg\,min}
\DeclareMathOperator*{\argmax}{arg\,max}
\DeclareMathOperator*{\arginf}{arg\,inf}
\usepackage{bm}
\newcommand{\mname}{CIMAGE}

\AtBeginDocument{%
  }

\copyrightyear{2025}
\acmYear{2025}
\setcopyright{rightsretained}
\acmConference[WSDM '25]{Proceedings of the Eighteenth ACM International Conference on Web Search and Data Mining}{March 10--14, 2025}{Hannover, Germany}
\acmBooktitle{Proceedings of the Eighteenth ACM International Conference on Web Search and Data Mining (WSDM '25), March 10--14, 2025, Hannover, Germany}
\acmDOI{10.1145/3701551.3703515}
\acmISBN{979-8-4007-1329-3/25/03}

\makeatletter
\gdef\@copyrightpermission{
 \begin{minipage}{0.3\columnwidth}
  \href{https://creativecommons.org/licenses/by/4.0/}{\includegraphics[width=0.90\textwidth]{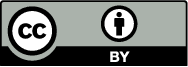}}
 \end{minipage}\hfill
 \begin{minipage}{0.7\columnwidth}
  \href{https://creativecommons.org/licenses/by/4.0/}{This work is licensed under a Creative Commons Attribution International 4.0 License.}
 \end{minipage}
 \vspace{5pt}
}
\makeatother

\begin{document}

\title{\texorpdfstring{\mname{}: Exploiting the Conditional Independence \\ in Masked Graph Auto-encoders}{\mname: Exploiting the Conditional Independence in Masked Graph Auto-encoders}}


\author{Jongwon Park}
\orcid{0009-0004-6627-4343}
\authornote{Equal contribution.}
\affiliation{
  \institution{Sungkyunkwan University}
  \city{Suwon}
  \country{Republic of Korea}
}
\email{jongwon208@skku.edu}

\author{Heesoo Jung}
\orcid{0000-0002-6554-2391}
\authornotemark[1]
\affiliation{
  \institution{Sungkyunkwan University}
  \city{Suwon}
  \country{Republic of Korea}
}
\email{steve305@skku.edu}

\author{Hogun Park}\authornote{Corresponding author.}
\affiliation{
  \institution{Sungkyunkwan University}
  \city{Suwon}
  \country{Republic of Korea}
}
\email{hogunpark@skku.edu}

\renewcommand{\shortauthors}{Jongwon Park, Heesoo Jung, and Hogun Park}






\begin{abstract}
Recent Self-Supervised Learning (SSL) methods encapsulating relational information via masking in Graph Neural Networks (GNNs) have shown promising performance. However, most existing approaches rely on random masking strategies in either feature or graph space, which may fail to capture task-relevant information fully. We posit that this limitation stems from an inability to achieve \textit{minimum redundancy} between masked and unmasked components while ensuring \textit{maximum relevance} of both to potential downstream tasks. Conditional Independence (CI) inherently satisfies the \textit{minimum redundancy and maximum relevance} criteria, but its application typically requires access to downstream labels. To address this challenge, we introduce \mname{}, a novel approach that leverages Conditional Independence to guide an effective masking strategy within the latent space. \mname{} utilizes CI-aware latent factor decomposition to generate two distinct contexts, leveraging high-confidence pseudo-labels derived from unsupervised graph clustering. In this framework, the pretext task involves reconstructing the masked second context solely from the information provided by the first context. Our theoretical analysis further supports the superiority of \mname{}'s novel CI-aware masking method by demonstrating that the learned embedding exhibits approximate linear separability, which enables accurate predictions for the downstream task. Comprehensive evaluations across diverse graph benchmarks illustrate the advantage of \mname{}, with notably higher average rankings on node classification and link prediction tasks.
Notably, our proposed model highlights the underexplored potential of CI in enhancing graph SSL methodologies and offers enriched insights for effective graph representation learning.
\end{abstract}


\maketitle

\section{INTRODUCTION}
Self-Supervised Learning (SSL) methods have emerged as a potent strategy to alleviate the need for a large amount of labeled data ~\cite{Chen20231, Wu2023, Chen20232}. Intensive research has been conducted in SSL, including masked modeling, from both theoretical and empirical perspectives, introducing numerous pretext tasks such as in-painting ~\cite{pathak2016context} and colorization ~\cite{Sordoni2021, Zhang2016}.  The framework of leveraging reconstruction-based pretext tasks has been successfully adapted to the graph domain, as highlighted in the comprehensive study by ~\cite{zhou2023comprehensive}, leading to the development of masked graph auto-encoders based on their empirical performance~\cite{MaskGAE, Hou2023, Li2023,graphbalance_www2025}. 

Although masked graph auto-encoders have demonstrated competitive performance among graph SSL baselines, the random masking method often results in sub-optimal representations for certain downstream tasks, particularly when the pretext task of reconstructing the masked portions is irrelevant to those tasks~\cite{ShiDTLL23}.
Effective masking strategies in SSL are generally designed to produce embeddings from masked input data that exhibit high Mutual Information (MI) with the downstream labels while maintaining low MI with the original input data~\cite{Tosh2021}. However, random masking alone may not consistently meet this requirement. An embedding that satisfies these conditions is considered to adhere to the minimum Redundancy and Maximum Relevance (mRMR) criterion ~\cite{Peng2005}. The criterion provides a theoretical foundation to maximize the relevance of the embedding for downstream tasks while minimizing unnecessary duplication of information from the input data. However, the design of the pretext task and the practical estimation of the mRMR in graph SSL remains underexplored.

To estimate the mRMR property more effectively, Conditional Independence (CI) has been explored in previous theoretical studies~\cite{Lee2021, Tosh2021}. These studies suggest that if a set of features can be decomposed into two subsets that satisfy the CI condition, then masking either subset constitutes a valid and effective masking strategy.
Despite its fundamental role in extracting information relevant to downstream tasks, CI is not directly addressed in many SSL studies~\cite{Tsai2021, Zhang2021} since the labeled data is mandated to satisfy the CI property~\cite{DBLP:journals/jstsp/ZaiemPEH22}.
In the graph domain, recent advancements in graph clustering have demonstrated the viability of treating pseudo-labels as potential downstream labels, evidenced by their superior performance across a range of downstream tasks~\cite{peng2022gate,DBLP:conf/aaai/YangLZWTZLFZ23,Graph-cluster-link}. 

In this paper, we propose a Conditional Independence Aware Masked Graph Auto-Encoder (\mname{}). This novel framework explicitly integrates the CI property using high-confidence pseudo-labels generated by unsupervised graph clustering algorithms. Our approach aims to advance the theoretical foundations of masking strategies in graph SSL. To identify subsets of encoded representations that satisfy the CI property, we decompose the graph representation into $K$ components.
This decomposition is guided by a Hilbert-Schmidt Independence Criterion-based score, which ensures minimum redundancy and maximum relevance. 

Central to our approach is a novel pretext task that involves masking and reconstructing contexts from one of the CI-compliant subsets. Additionally, we incorporate information on the graph's structural properties by reconstructing edges from the edge-masked graph. The \mname{} framework is optimized using three distinct loss functions: clustering loss, structure reconstruction loss, and latent factor reconstruction loss.

Our theoretical analysis further substantiates the effectiveness of \mname{}’s novel masking method by showing that the learned embeddings achieve approximate linear separability. This property enables accurate predictions in the downstream task, underscoring the robustness and effectiveness of our approach. An extensive experimental evaluation across seven leading benchmarks has been conducted to validate the effectiveness of our proposed framework. Notably, \mname{} has significantly improved node classification and link prediction tasks. 

\section{RELATED WORK}

\subsection{Self-Supervised Learning on Graphs}

Self-Supervised Learning (SSL) methods, including contrastive learning and reconstruction-based approaches, have been effectively applied to graphs, often employing Graph Neural Networks (GNNs) as encoders. Contrastive learning in GNNs generates augmented graph views and compares their representations to differentiate between positive and negative pairs, as illustrated in ~\cite{dgi} and ~\cite{infograph}. The success of reconstruction-based SSL in the image domain has significantly impacted the graph SSL, mainly through adopting a masking method that exploits the reconstruction capabilities of Graph Auto-Encoders (GAEs) ~\cite{Kipf2016}. Several studies have implemented random node feature masking as a pretext task for node feature reconstruction, moving beyond traditional edge reconstruction ~\cite{Li2023, Hou2023}. Recent advancements have introduced latent space masking to generate more generalized representations~\cite{ShiDTLL23}.
However, the widespread reliance on random masking in graph SSL presents significant challenges. First, solely depending on the random masking could alter the underlying semantics of graphs ~\cite{Lee2022, Jin2021}. Second, it fails to ensure minimum redundancy between the components of the pretext task to ensure the upper bound of SSL ~\cite{Tosh2021}.

\subsection{Conditional Independence in SSL}
Exploring Conditional Independence (CI) in self-supervised learning (SSL) provides a theoretical framework for understanding how pretext tasks can enhance representation learning, ultimately improving performance on downstream tasks. Despite its potential, the direct application of CI properties is often avoided due to the necessity for labeled data, which poses a challenge. Studies such as those by ~\cite{Tsai2021, Tosh2021} have implicitly incorporated CI principles, focusing on assumptions related to CI and redundancy within multi-view settings. Similarly, ~\cite{Teng2022} provides guarantees for the quality of representations learned through masked auto-encoder SSL techniques, while ~\cite{Lee2021} focuses on investigating broader and more practical scenarios, establishing error margins for SSL concerning CI.
Several previous studies have utilized the concept of CI in the context of graph SSL. For instance, ~\cite{xiao2022decoupled} tackles non-homophilous graphs by implicitly leveraging CI principles between a target node and its neighbors, conditioned on the input graph structure. However, due to the typical requirement for labels in CI evaluation, most existing methods assess learned representations at the evaluation phase without explicitly incorporating CI during the learning process.

\section{PRELIMINARY}
 The core challenge in \mname{} lies in effectively masking contexts generated by subsets of latent variables that meet specific Conditional Independence (CI) requirements. We employ two key methodologies to address this challenge: the minimum Redundancy Maximum Relevance (mRMR) and Hilbert-Schmidt Independence Criterion (HSIC). The mRMR approach, detailed in Section 3.1, facilitates the selection of latent variables to be masked by considering the CI between subsets of variables given the pseudo-label. The Hilbert-Schmidt Independence Criterion in Section 3.2 estimates mutual information between latent variables. 
 
\subsection{Minimum Redundancy Maximum Relevance}
The exploration of non-linear relationships between random variables has long been a topic of interest since the advent of the Least Absolute Shrinkage and Selection Operator (LASSO) methodology ~\cite{Tibshirani1996}. 
The criterion of minimum Redundancy Maximum Relevance (mRMR) ~\cite{Peng2005} promotes the selection of non-redundant features highly relevant to the output by elucidating such non-linear relations.
This criterion prioritizes the relevance between input and output and between different inputs.
The mRMR approach examines the independence of two discrete random variables, $X$ and $Y$
Therefore, independence can be assessed without knowledge of the distributions $P_{X}, P_{Y}, P_{XY}$. The mRMR method can be reformulated as an iterative process of selecting features, as outlined below:
\begin{equation}
\argmax_{\bm{f}_i \in X} I(c;\bm{f}_{i})-\frac{1}{\lvert S \rvert}\sum_{\bm{f}_{s} \in S}{I(\bm{f}_{s};\bm{f}_{i})},
\label{eq:mrmr}
\end{equation}
where $\bm{f}_{i}$ denotes each candidate feature, $S$ represents the set of selected features $\bm{f}_{s}$, $c$ stands for the class label, and $I$ is the mutual information measuring relevance. While Equation \eqref{eq:mrmr} is widely accepted and used, it does not consider class-relevant redundancy. To overcome this shortcoming and fully extract features that satisfy minimum redundancy, ~\cite{Lin2006} introduces the Conditional Informative Feature Extraction (CIFE) objective:
\begin{equation}
\vspace{0.5mm}
\argmax_{\bm{f}_{i}\in X}[I(c;\bm{f}_{i}) - \sum_{\bm{f}_{s}\in S}{I(\bm{f}_{s};\bm{f}_{i})} + \sum_{\bm{f}_{s}\in S} {I(\bm{f}_{s};\bm{f}_{i}\! \mid \! c)}].
\label{eq:CIFE}
\end{equation}

\subsection{Hilbert–Schmidt Independence Criterion}
One challenge of Equation ~\eqref{eq:CIFE} is that the mutual information between features is intractable.
The Hilbert-Schmidt Independence Criterion (HSIC) has emerged as an efficient and intuitive statistical tool to measure the independence between multi-dimensional random variables, where HSIC is 0 if and only if two random variables are independent. Consider $P_{\bm{x}\bm{y}}$ as a Borel probability measure on the domain $\mathbb{X} \times \mathbb{Y}$, where $\mathbb{X}$ and $\mathbb{Y}$ represent separable metric spaces.

Suppose there is an i.i.d sample of length $m$: $(\bm{X},\bm{Y}) = \{(x_i,y_i)\}^{m}_{i=1}$ drawn from $P_{xy}$ with $(x,x') \in \mathbb{X}$, and $(y,y') \in \mathbb{Y}$, where $x'$ and $y'$ denote independent copies of $x$ and $y$ respectively. There exists a linear feature map $\phi:\mathbb{X} \rightarrow \mathbb{V}$ with associated positive definite (p.d.) kernel $k: \mathbb{X} \times \mathbb{X} \rightarrow \mathbb{R}$ satisfying $k(x,x')= \langle \phi(x), \phi(x') \rangle_\mathbb{V}$. In the same manner, there exists $\varphi:\mathbb{Y} \rightarrow \mathbb{G}$ with corresponding p.d. kernel $l$, satisfying $l(y,y')= \langle \varphi(y), \varphi(y') \rangle_\mathbb{G}$. Consequently, the HSIC can be written as:
\begin{equation}
\begin{split}
&\operatorname{HSIC}(\bm{X},\bm{Y}) =\mathbb{E}_{xx'yy'}[k(x,x')l(y,y')]\\
&-2\mathbb{E}_{xy}\left[\mathbb{E}_{x'}[k(x,x')]\mathbb{E}_{y'}[l(y,y')]\right]  \\
&+\mathbb{E}_{xx'}[k(x,x')]\mathbb{E}_{y,y'}[l(y,y')].
\end{split}
\end{equation}
An empirical approximation method has been introduced to transform the HSIC into a practical measure for statistical independence testing ~\cite{Gretton2005}. Here, $\mathrm{Tr}$ denotes the trace of a matrix:
\begin{equation}
\widehat{\operatorname{HSIC}}(\bm{X},\bm{Y}) = \frac{1}{(n-1)^2}\mathrm{Tr}(k(x,x'),l(y,y')).
\label{eq:HSIC-trace}
\end{equation}

\section{METHODS}
\begin{figure*}[t]
\centering
\includegraphics[width=\textwidth]{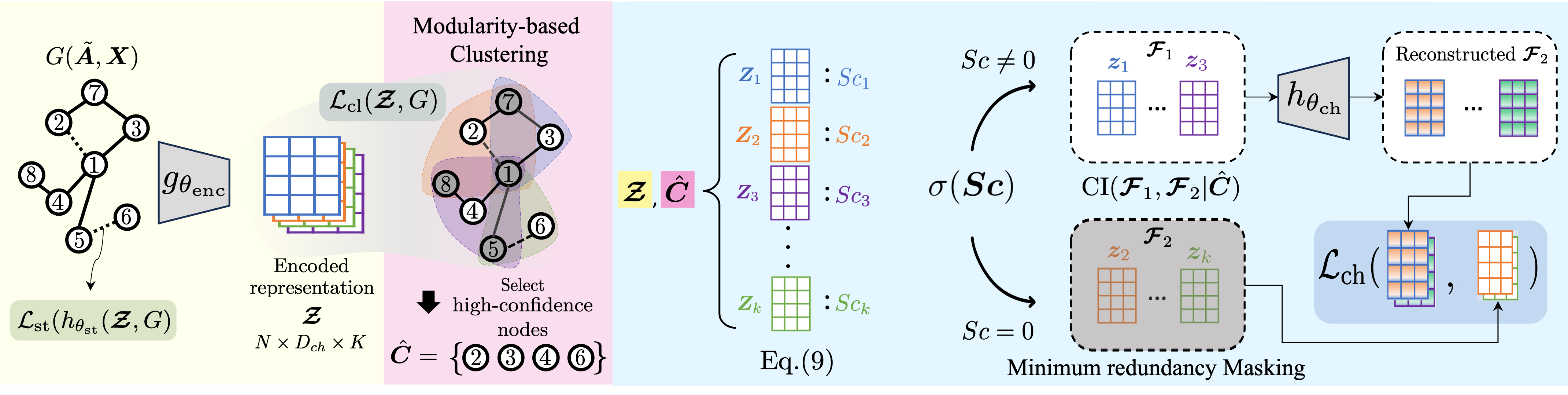}
\caption{The figure illustrates the overall architecture of \mname{}. It identifies \(K\) explanatory latent factors from the underlying structure of a masked input graph \(G(\tilde{\bm{A}}, \bm{X})\). The encoded representation extracts high-quality pseudo label $\hat{\bm{C}}$ by performing modularity-based clustering.
Each latent factor, represented by \(\bm{Z}_k\), is assigned with the scored $Sc_k$ according to Equation (\ref{eq:calculateSc}). The factors are then partitioned into two distinct contexts, \(\bm{\mathcal{F}}_1\) and \(\bm{\mathcal{F}}_2\). These contexts satisfy conditional independence to the given pseudo labels and serve as the input and target for the context reconstruction function \(h_{\theta_{\text{ch}}}\).}
\vspace{-2mm}
\label{fig:overall}
\end{figure*}

\subsection{Notations}
Throughout this section, we consider an undirected graph $G$=$(\tilde{\bm{A}},\bm{X})$. Here, \(\tilde{\bm{A}} \in \{0,1\}^{N \times N}\) denotes an initial edge-masked adjacency matrix (where $N$ is the number of nodes), and \(\bm{X} \in \mathbb{R}^{N \times D_{node}}\) represents node features matrix. $\mathcal{V}$ and $\mathcal{E}$ denote the sets of nodes and edges in $G$, respectively. Every node \(v \in \mathcal{V}\) possesses a \(D_{node}\)-dimensional feature vector \(\bm{x}_v\). We use \(K\) to denote the number of latent factors and \(D_{ch}\) to specify the dimension of a latent factor. The decomposed representation for conditional independence of each latent factor is denoted by \(\bm{Z}_{k} \in \mathbb{R}^{N \times D_{ch}}\), while the comprehensive encoded representation of the graph \(G\) is represented as \(\bm{\mathcal{Z}} \in \mathbb{R}^{N \times D_{ch} \times K}\). The conditional independence is satisfied with the pseudo-label $\hat{\bm{C}}$, which is acquired from graph-clustering.

\subsection{Overview}
Figure \ref{fig:overall} shows that \mname{} employs a CI-aware auto-encoder architecture. The encoder, denoted as \(g_{\theta_{\text{enc}}}\), generates the encoded representation $\bm{\mathcal{Z}}$, which is utilized by two primary components: 
a modularity-based clustering and a minimum redundancy masking module. The modularity-based clustering loss partitions the input graph to produce high-confidence pseudo-labels. Subsequently, the minimum redundancy masking module leverages these pseudo-labels and the encoded representation to mask a subset of latent factors that satisfy minimum redundancy. To be specific, the encoded representation $\bm{\mathcal{Z}}$ is partitioned into two sub-tensors, \(\bm{\mathcal{Z}}=[\bm{\mathcal{F}_1}\! \mid \! \bm{\mathcal{F}_2}]\), such that they satisfy $\text{CI}(\bm{\mathcal{F}_1}, \bm{\mathcal{F}_2}|\hat{\bm{C}})$. 
The CI-aware masking reconstruction loss uses the subset $\bm{\mathcal{F}_1}$ as the input, while $\bm{\mathcal{F}_2}$ serves as the target for the context reconstruction function. This approach ensures that masking $\bm{\mathcal{F}_2}$ corresponds to the minimal redundancy-aware masking strategy. Additionally, a structure reconstruction loss is employed to capture a deeper understanding of structural proximity information within the encoded representation.

The model uses hyper-parameters $\lambda_1$ and $\lambda_2$ to balance the three losses: the structure reconstruction loss $\mathcal{L}_\text{st}$ for reconstructing masked edges, the latent feature (channel) reconstruction loss $\mathcal{L}_\text{ch}$ for reconstructing masked context, and clustering loss $\mathcal{L}_{\text{cl}}$ to enhance the clustering performance respectively. The joint reconstruction in both the latent factor space and the raw data space simultaneously leverages topological similarity information and minimizes redundancy. The following equation is used to optimize \mname{}: 

\begin{equation}
\begin{split}
\theta_{\text{enc}}^\ast, \theta_{\text{st}}^\ast, \theta_{\text{ch}}^\ast 
&= \argmin_{\theta_{\text{enc}}, \theta_{\text{st}}, \theta_{\text{ch}}} \mathcal{L}_{\text{st}}(h_{\theta_{\text{st}}}(g_{\theta_{\text{enc}}}({G})), G) \\
&\quad + \lambda_1 \mathcal{L}_{\text{ch}}(h_{\theta_{\text{ch}}}(\bm{\mathcal{F}}_1),\bm{\mathcal{F}}_2) \\
&\quad + \lambda_2 \mathcal{L}_{\text{cl}}(g_{\theta_{\text{enc}}}({G}), G).
\end{split}
\end{equation}
The optimized encoder, denoted as $g_{\theta_\text{enc}^{\ast}}$, remains fixed and is subsequently employed for the downstream tasks.

\subsection{Encoder}
The prevalent approach in graph auto-encoders utilizing masking techniques \cite{MaskGAE, Hou2023, Tan2023, Li2023} typically employs conventional message propagation layers such as GAT \cite{gat}. While this strategy is feasible, it inherently works in a holistic manner and is unsuitable for dividing the latent factors to satisfy CI.
To overcome these limitations, our encoder architecture incorporates a decomposing algorithm that aims to uncover the underlying factors of the graph. For the experiment, we apply disentanglement \cite{sun2019vgraph, jianxin2019, li2022disentangled, li2021disentangled} for decomposing the graph. A CI score $Sc_k$ is then calculated with the decomposed representations to ensure the preservation of CI among context, a subset of latent factors.

\subsubsection{Minimum-Redundancy Masking Module.}
\label{subsubsec:mR-mask}
The primary objective is to assign a CI score, $Sc_k$, to each latent factor for minimal redundant masking. We calculate $Sc_k$ via Bayesian Lasso to ensure CI between context formed by factors with non-zero scores and the remaining context, given the pseudo-label $\hat{\bm{C}}$. 

Given $\tilde{\mathbf{A}}$ and $\mathbf{X}$, we apply disentanglement~\cite{jianxin2019} to produce $\mathbf{Z}_k$, where $\mathbf{Z}_k$ is one of the $K$ latent factors, $\mathbf{Z}_k = \mathbf{\mathcal{Z}}_{[:,:,k]}$. Each node features $\bm{x}_v$ is projected into $K$ distinct sub-spaces using MLP and denote it as $l_{v,k} \in \mathbb{R}^{D_{ch}}$. 
The approach we employ to acquire the disentangled representation can be written as:
\begin{equation}\label{eq:neigh}
\bm{Z}_{v,k}^{t} =\frac{\bm{l}_{v,k} + \sum_{u}{\text{softmax}(\bm{l}_{u,k}^{\hspace{1.0em}\mathsf{T}}\bm{Z}_{u,k}^{(t-1)})}\bm{l}_{u,k}}{\lVert \bm{l}_{v,k} + \sum_{u}{\text{softmax}(\bm{l}_{u,k}^{\hspace{1.0em}\mathsf{T}}\bm{Z}_{u,k}^{(t-1)})}\bm{l}_{u,k}\rVert_2},
\end{equation}
where $\bm{Z}_{v,k}^{0}=l_{v,k}$ sets the starting point for each node's latent factor representation, and $u$ is the neighbor of the node $v$.
The encoded representation $\bm{\mathcal{Z}}$ is achieved by stacking the outcomes $\bm{Z}_{v,k}^{t}$. Leveraging the encoded representation $\bm{Z}_{v,k}^{t}$ enables us to analyze the underlying elements that influence node interactions. 

The empirical estimation of HSIC is employed to measure independence at the simpler version of CIFE by ~\cite{gao2020} and solve the optimization problem relaxed by $\bm{\omega} \in \mathbb{R}^{D}_{+}$, where $D=D_{ch} \cdot K$. To clarify notation, $\beta$ is a regularization parameter, $\bm{f}_i \in \mathbb{R}^{N}$ represents a single feature of $\bm{Z}_{1:K} \in \mathbb{R}^{N \times D}$, and $S$ denotes the selected subset of features. By transforming mutual information in Eq.~\eqref{eq:CIFE} into HSIC, it is reformulated as:
\begin{equation}
\begin{aligned}
&\argmax_{\bm{\omega}} \sum_{i=1}^{D}{\omega_{i}\widehat{\operatorname{HSIC}}(\bm{f}_{i},\hat{\bm{C})}} \\
& - \sum_{i,j=1 }^{D}{\omega_{i}\omega_{j}[\widehat{\operatorname{HSIC}}(\bm{f}_i,\bm{f}_j)-\widehat{\operatorname{HSIC}}(\bm{f}_i,\bm{f}_j|\hat{\bm{C}})]} - \beta\lvert\bm{\omega}\rvert_1.
\label{eq:7}
\end{aligned}
\end{equation}
The first term promotes selecting channels ($f_i$) highly relevant to pseudo-labels, while the second minimizes redundancy, allowing some when strongly relevant to the target.


With $\widehat{\operatorname{HSIC}}(\bm{f}_i,\bm{f}_s\! \mid \! \hat{\bm{C}})$ regarded as a weighted mean that considers the number of samples for each pseudo-label class, 
the Conditional Mutual Information-HSIC (CMI-HSIC) can be articulated as follows:
\begin{equation}
\argmax_{\bm{\omega}} \sum_{i=1}^{D}{[\omega_i \widehat{\operatorname{HSIC}}(\bm{f}_i,\hat{\bm{C}})]}-(1-\pi)\sum_{i,j=1}^{D}{[\omega_i \omega_j \widehat{\operatorname{HSIC}}(\bm{f}_i,\bm{f}_j)]}- \beta\lvert\bm{\omega}\rvert_1,
\label{eq:CMIHSIC}
\end{equation}
where $\pi$ is the coefficient representing the weighted mean. 
The exact method for obtaining $\bm{\omega}$ is detailed in the Appendix.

The learned parameters $\omega_1^*, \omega_2^*, \ldots, \omega_{D}^*$ are averaged per factor by $\bar{\omega}_k^* = \frac{1}{D_{ch}}\sum^{kD_{ch}}_{(k-1)D_{ch}+1}{\omega_k^*}$. The normalized $\bar{\omega}^*_k$ through Equation (\ref{eq:calculateSc}) is used as $Sc_k$. $\operatorname{max}(\bm{\bar{\omega}^*})$ represents the maximum values among the averaged values $\bar{\omega}^*_k$.
\begin{equation}
Sc_k = \frac{\bar{\omega}^*_k}{\operatorname{max}(\bm{\bar{\omega}^*})}.
\label{eq:calculateSc}
\end{equation}

With each factor $k$ assigned its score $Sc_k$, applying the piecewise function $\sigma$ facilitates the precise classification of latent factors into distinct contexts. 

\[
\sigma(Sc_k) = 
\begin{cases} 
      \bm{\mathcal{F}}_1 & \text{if } Sc_k \text{ is non-zero} \\
      \bm{\mathcal{F}}_2 & \text{if } Sc_k \text{ is zero}
\end{cases}
\]

This process is crucial for delineating subsets of latent factors that satisfy CI. For every latent factor represented by vector $\bm{Z}_k$, the function $\sigma$ evaluates its associated score $Sc_k$. A latent factor with a non-zero score $Sc_k$ is classified into the first context denoted as $\bm{\mathcal{F}}_1$. Conversely, remaining latent factors, where $Sc_k$ corresponds to 0, are used to form an alternative context $\bm{\mathcal{F}}_2$. This binary classification mechanism effectively partitions (masks) the latent factors, thereby meeting the CI objective. Furthermore, $\beta$ in Eq.~\eqref{eq:7} facilitates the control of the number of zeros (constructing $\bm{\mathcal{F}}_2$) through L1 regularization.

\subsubsection{Clustering loss}

Conditional independence typically requires labels, which conflicts with the self-supervised learning approach. However, graph clustering, which effectively predicts clusters without relying on labels, has demonstrated outperforming performance. Consequently, we employ pseudo-labels derived from graph clustering~\cite{XZYLL23,MrabahBTK23}. Moreover, we refine our encoder to enhance the modularity property of a graph. Modularity~\cite{modularity}, denoted as $Q$, measures the difference between the fraction of the edges within the clusters and the expected fraction assuming random clusters and can be expressed as:
\begin{equation}
    Q = \frac{1}{2m}\sum_{ij}(\tilde{\bm{A}}_{ij} - \frac{d_i d_j}{2m}\delta(c_i, c_j)),
\end{equation}
where $d_i$ is the degree of node $i$ and $\delta(c_i, c_j)$ is the indicator function, $\delta(c_i, c_j)$ is 1 when the clusters of node $i$ and $j$ are identical and 0 otherwise, and $m$ is the total number of edges. Using the learned representation $\bm{Z}_{1:K}$, $\delta(c_i, c_j)$ can be approximated, which leads to our clustering loss as follows:
\begin{equation}
    \mathcal{L}_{\text{cl}} = -Q = -\frac{1}{2m}Tr((\tilde{\bm{A}} - \frac{d d^T}{2m})\bm{Z}_{1:K}\bm{Z}_{1:K}^T),
\end{equation}
where $\mathrm{Tr}(\bm{A})$ is the trace of matrix $\bm{A}$. If the mean value of the approximated indicator function for each node exceeds a high-confidence threshold of above 0.99, it is utilized as the pseudo-labeled node. Additionally, to ensure efficiency in large graphs, the approximated indicator function is generated using mini-batch matrix multiplication~\cite{MAGI}.

\subsection{Decoder}

Our approach utilizes two types of masks: the initial edge mask and the latent-factor mask. 
Although the initial edge mask is involved in the process, our primary focus is on the latent-factor mask. 
Consequently, our decoder is mainly designed to enhance the graph representation by integrating the contexts derived from latent factors. Together with the edge-masked reconstruction, these approaches mitigate the issue of information redundancy, which often hampers the efficacy of graph representation.


\subsubsection{Latent factor reconstruction loss.}
In Section~\ref{subsubsec:mR-mask}, a minimum redundancy masking module is proposed to extract encoded representation $\bm{\mathcal{Z}}$, which can be separated into two distinctive expressive representations $\bm{\mathcal{F}}_1, \bm{\mathcal{F}}_2$ based on sets of latent factors and corresponding CI scores. $\bm{\mathcal{F}}_1$ consists of latent factors with non-zero $Sc_k$ and is utilized as an input for the latent factor reconstruction function. The goal is to predict a masked representation $\bm{\mathcal{F}}_2$, using the remaining subset of latent factors $\bm{\mathcal{F}}_1$ that meet the CI criterion conditioned by pseudo-labels. Optimizing with the proposed loss function aligns with achieving the SSL objective of minimum redundancy. 


The latent factor reconstruction function, $h_{\theta_{\text{ch}}}$ employs a simple MLP structure and is trained using the scaled cosine error (SCE) loss function introduced by ~\cite{Lin2017}. It has been widely adopted in various works, including ~\cite{Nie2023, Hou2023}. 

\begin{equation}
\mathcal{L}_\text{ch} = \frac{1}{\lvert \mathcal{V} \rvert} \sum_{v \in \mathcal{V}}{(1-\frac{\bm{\mathcal{F}}_{2[v,:,:]}^{\hspace{1.0em}\mathsf{T}}h_{\theta_{\text{ch}}}(\bm{\mathcal{F}}_{1[v,:,:]})}{\lVert \bm{\mathcal{F}}_{2[v,:,:]} \rVert \cdot \lVert h_{\theta_{\text{ch}}}(\bm{\mathcal{F}}_{1[v,:,:]}). \rVert})^{\tau}},
\end{equation}
where $\tau$ is a hyper-parameter to scale the latent factor reconstruction loss function.
The SCE loss can be viewed as an adaptive sample re-weighting method to address challenges arising from the multi-dimensional and continuous nature of $\bm{\mathcal{F}}_2$. 

\subsubsection{Structure reconstruction loss.}
The structure reconstruction function $h_{\theta_{\text{st}}}$ is designed to reconstruct masked raw edge data, adhering to the fundamental design principles of GAEs. This process involves aggregating pairwise node representations to construct link representations, allowing a transformation of individual node information into a comprehensive understanding of their interactions. The corresponding edge reconstruction loss function, denoted as $\mathcal{L}_\text{st}$, is defined as follows: 
\begin{equation}
\begin{split}
&\mathcal{L}_\text{st}^+ = \frac{1}{\lvert \mathcal{E}^+ \rvert}\sum_{(u,v)\in\mathcal{E}^+}{\log h_{\theta_{\text{st}}} (\bm{\mathcal{Z}}_{u,:,:},\bm{\mathcal{Z}}_{v,:,:})}  \\
&\mathcal{L}_\text{st}^- = \frac{1}{\lvert \mathcal{E}^- \rvert}\sum_{(u',v') \in \mathcal{E}^-} {\log h_{\theta_{\text{st}}} (\bm{\mathcal{Z}}_{u^{\prime},:,:},\bm{\mathcal{Z}}_{v^{\prime},:,:})} \\
&\mathcal{L}_\text{st} = -(\mathcal{L}_\text{st}^{+} + \mathcal{L}_\text{st}^{-})
\end{split}
\end{equation}
$\mathcal{E}^+$ denotes the set of existing edges, while $\mathcal{E}^-$ signifies an equal number of negative edges set. The structure decoder $h_{\theta_{\text{st}}}$ is defined as below with $\odot$ denoting element-wise multiplication where high values correspond to a high probability of an edge between two nodes:
\begin{equation}
h_{\theta_{\text{st}}} (\bm{\mathcal{Z}}_{u,:,:},\bm{\mathcal{Z}}_{v,:,:})= \sigma(\operatorname{MLP}(\bm{\mathcal{Z}}_{u,:,:} \odot \bm{\mathcal{Z}}_{v,:,:})).
\end{equation}

\subsection{Theoretical Analysis}
\subsubsection{Extracting task-relevant information.}
Embeddings generated by SSL are considered effective when they achieve minimum redundancy. Suppose $\bm{\mathcal{F}}_1$ and $\bm{\mathcal{F}}_2$ represent two different views of the graph $G$.
Based on the conditional independence property, $\bm{\mathcal{F}}_1$ and $\bm{\mathcal{F}}_2$ are sufficient to represent the graph $G$ conditioned on the pseudo label.
Therefore, two representations are expected to conserve only the information required for downstream tasks, referred to as task-relevant information, which leads to $I(\bm{\mathcal{F}}_1;\hat{\bm{C}})=I(\bm{\mathcal{F}}_2;\hat{\bm{C}})$
Any shared information between these two views exceeding $I(G;\hat{\bm{C}})$ can be deemed task-irrelevant, hindering the downstream task performance. 
$\bm{\mathcal{F}}_2$ is a self-supervised signal due to the fact that the target signal of $\mathcal{L}_{\text{ch}}$ is $\bm{\mathcal{F}}_2$. 
With the assumption that the pseudo-label $\hat{\bm{C}}$ resembles true-label $\bm{C}$ and \(I(\bm{\mathcal{F}}_1; \bm{C} \! \mid \! \bm{\mathcal{F}}_2)\) has a marginal potential error, the maximum relevance of $\bm{\mathcal{F}}_1$ and $\bm{\mathcal{F}}_2$ with true label $\bm{C}$ can be easily inferred from the prior work~\cite{Tsai2021,xiao2022decoupled}.
Therefore, the learned contexts $\bm{\mathcal{F}}_1$ and $\bm{\mathcal{F}}_2$ satisfy maximum relevance to the true downstream label by optimizing $\mathcal{L}_{\text{ch}}$.

\subsubsection{Learning optimal latent factor reconstruction function.}
We provide Theorem~\ref{theorem:optimaldecoder} to demonstrate that we can learn optimal latent factor reconstruction function $h_{\theta_{\text{ch}}^\ast}$ for reconstructing masked $\bm{\mathcal{F}}_2$. 
\begin{restatable}{theorem}{theorema}
\textup{(Optimal approximation.)} If $\operatorname{HSIC}(\bm{\mathcal{F}}_1,\bm{\mathcal{F}}_2\! \mid \!\hat{\bm{C}})=0, dim(\bm{\mathcal{F}}_2) > \lvert\hat{\bm{C}}\rvert$, and  $rank(\mathbb{E}[\bm{\mathcal{F}}_2 \! \mid \! \hat{\bm{C}}=c])=\lvert \hat{\bm{C}} \rvert$, then downstream task labels are guaranteed to be linearly separable with 
optimal structural decoder $h^*_{\theta_{\text{ch}}}(\bm{\mathcal{F}}_1)$, and its complexity reduces to $\tilde{O}(dim(\hat{\bm{C}}))$.
\label{theorem:optimaldecoder}
\end{restatable}
Detailed proofs of the theorem can be found in the Appendix. Theorem \ref{theorem:optimaldecoder} states linear classification on $\hat{\bm{C}}$ can be done with $h^*_{\theta_{\text{ch}}}(\bm{\mathcal{F}}_1)$, and downstream sample complexity decreases to the dimension of $\hat{\bm{C}}$. The downstream sample complexity should reduce from $\tilde{O}(\lvert n_2\rvert)$ to a value between $\tilde{O}(\lvert \hat{\bm{C}}\rvert)$ and $\tilde{O}(\lvert D_{ch} \rvert)$ under the assumption of CI and optimal $h^*_{\theta_{\text{ch}}}$, where $n_2$ denotes labeled data \cite{Lee2021, Teng2022}. 

In our approach, we leverage both $\bm{\mathcal{F}}_1$ and $\bm{\mathcal{F}}_2$ by constructing the feature set $\bm{\mathcal{Z}} = [ \bm{\mathcal{F}}_1 | \bm{\mathcal{F}}_2 ]$ for downstream tasks rather than relying solely on $\bm{\mathcal{F}}_2$. This strategy is justified for two main reasons: First, while the mRMR aims to ensure that the mutual information $I(\mathcal{Z};\mathcal{C})$ equals $I(G;\bm{C})$, achieving perfect CI in practice is challenging due to inevitable errors. Second, since $\bm{\mathcal{F}}_2$ consists of a limited number of factors, fewer than $K$, as shown in Table~\ref{tab:minredun}, the performance might be less effective.
 These two practical limitations may prevent the extraction of all task-relevant information necessary for effective downstream performance. 


\begin{table*}[t]
\caption{Node classification accuracy (\%) on seven benchmark datasets. In each column, the \textbf{boldfaced} score denotes the best result, and the \underline{underlined} score represents the second-best result. A.R. refers to the average ranking. OOM means out of memory.}
\centering
\begin{tabular}{l|c|c|c|c|c|c|c|c}
\hline
Model & Cora & Citeseer & Photo & Computers & Pubmed & WikiCS & Arxiv & A.R. \\ \hline
GAE & 76.25 ± 0.15 & 63.89 ± 0.18 & 91.71 ± 0.08 & 86.33 ± 0.44 & 77.24 ± 1.10 & 73.44 ± 0.07 & 65.08 ± 0.24 & 12.71 \\
VGAE & 76.68 ± 0.17 & 64.34 ± 0.11 & 91.87 ± 0.04 & 86.70 ± 0.30 & 77.36 ± 0.60 & 75.12 ± 0.89 & 67.70 ± 0.03 & 11.14 \\
ARGA & 77.95 ± 0.70 & 64.44 ± 1.19 & 91.82 ± 0.08 & 85.86 ± 0.11 & 80.44 ± 0.74 & 69.61 ± 4.01 & 67.34 ± 0.09 & 11.86 \\
ARVGA & 79.50 ± 1.01 & 66.03 ± 0.65 & 91.51 ± 0.09 & 86.02 ± 0.11 & 81.51 ± 1.00 & 63.99 ± 1.71 & 67.43 ± 0.08 & 11.00 \\ \hline
DGI & 80.44 ± 0.64 & 67.74 ± 2.48 & 91.65 ± 0.46 & 87.45 ± 0.47 & 78.83 ± 0.77 & 77.52 ± 0.14 & 66.07 ± 0.45 & 10.43 \\
MVGRL & 81.22 ± 0.94 & 70.54 ± 0.85 & 92.64 ± 0.24 & 88.61 ± 0.64 & 79.46 ± 0.43 & 80.16 ± 0.53 & 69.10 ± 0.10 & 7.42 \\
GRACE & 81.90 ± 0.40 & 71.20 ± 0.50 & 92.15 ± 0.24 & 86.25 ± 0.25 & 80.60 ± 0.40 & 78.25 ± 0.65 & 69.80 ± 0.10 & 8.14 \\
CCA-SSG & 84.00 ± 0.40 & 73.10 ± 0.30 & 93.14 ± 0.14 & 88.76 ± 0.36 & 80.81 ± 0.38 & 78.85 ± 0.41 & 69.22 ± 0.22 & 4.71 \\ \hline
MaskGAE & 83.21 ± 0.70 & 72.62 ± 0.38 & 92.79 ± 0.18 & 89.36 ± 0.18 & 82.01 ± 0.52 & 76.38 ± 0.15 & 71.03 ± 0.17 & 5.14 \\
GraphMAE2 & 83.96 ± 0.85 & \underline{73.42 ± 0.30} & 92.60 ± 0.11 & 87.42 ± 0.52 & 81.23 ± 0.57 & 78.48 ± 0.12 & \underline{71.77 ± 0.14} & 5.14 \\ 
GiGaMAE & 82.13 ± 0.80 & 70.04 ± 1.07 & 93.01 ± 0.46 & \textbf{90.20 ± 0.45} & 80.55 ± 0.75 & \textbf{81.42 ± 0.35} & OOM & 6.43 \\ 
AUG-MAE & \underline{84.10 ± 0.55} & 73.16 ± 0.44 & 92.82 ± 0.17 & 88.52 ± 0.17 & 81.12 ± 0.53 & 71.81 ± 0.69 & 71.20 ± 0.30 & 5.42 \\
Bandana & 82.90 ± 0.39 & 71.39 ± 0.54 & \underline{93.40 ± 0.10} & 89.28 ± 0.14 & \underline{82.77 ± 0.49} & 78.63 ± 0.27 & 71.04 ± 0.39 & 4.14 \\
\hline
\mname{} & \textbf{84.40 ± 0.72} & \textbf{74.45 ± 0.32} & \textbf{93.54 ± 0.08} & \underline{89.39 ± 0.10} & \textbf{82.78 ± 0.67} & \underline{81.22 ± 0.31} & \textbf{72.89 ± 0.01} & \textbf{1.29} \\ \hline
\end{tabular}   
\label{tab:noderesult}
\end{table*}

A crucial aspect to ensure is that the rank of \(\mathbb{E}[\bm{\mathcal{F}}_2 \! \mid \! \hat{\bm{C}}=c]\) is \(\lvert \hat{\bm{C}} \rvert \), conforming that \(\bm{\mathcal{F}}_2\) correlates with every instance of \(\hat{\bm{C}}\) to avoid the trivial solution of $\bm{\mathcal{F}}_1 = \bm{\mathcal{F}}_2$. Instead of validating the MLP's rank, we alternatively demonstrate low redundancy and sample complexity empirically. Low redundancy can be quantified using the expression: 
\begin{align}
\mathbb{E}\left[(\mathbb{E}[\hat{\bm{C}} \mid \bm{\mathcal{F}}_1] - \mathbb{E}[\hat{\bm{C}} \mid \bm{\mathcal{F}}_1, \bm{\mathcal{F}}_2])^2\right] &= \notag \\
\mathbb{E}\left[(\mathbb{E}[\hat{\bm{C}} \mid \bm{\mathcal{F}}_2] - \mathbb{E}[\hat{\bm{C}} \mid \bm{\mathcal{F}}_1, \bm{\mathcal{F}}_2])^2\right] &= 0 \label{eq:F1=F2}.
\end{align}
This implies that predictions for \(\hat{\bm{C}}\) based on \(\bm{\mathcal{F}}_1\) or \(\bm{\mathcal{F}}_2\) individually should yield accuracy equivalent to predictions using both \(\bm{\mathcal{F}}_1\) and \(\bm{\mathcal{F}}_2\) concurrently, consistent with the finding of ~\cite{Tosh2021}. 

\subsection{Complexity Analysis}
The computational complexity of the masked graph auto-encoder framework is not a concern, as evidenced by the work of \cite{Hou2023}, where both the time and space complexities scale linearly with the number of nodes \( N \). Since \mname{}'s encoder and decoder consist of simple MLP, the computational cost would not be problematic. Therefore, we analyze the time complexity of calculating the CI score $Sc_k$.

The first step in the CI module requires  \( O(max(K,N_{\text{labeled\_train}})^2) \) to compute the HSIC score based on Eq.~\eqref{eq:7}, when $N_{\text{labeled\_train}}$ denotes the number of pseudo-labeled nodes and $K$ indicates the number of latent factors. The HSIC  between $\bm{f}_i$ and $\bm{f}_j$ corresponds to $O(K^2)$, and the HSIC between $f_i$ and $\hat{\bm{C}}$ corresponds to $O(max(K,$ $N_{\text{labeled\_train}})^2)$, since we need upsampling to compute HSIC. 
We note that we set $K$ to 16 or 32 in this paper.
The number of $N_{\text{labeled\_train}}$ or $K$ is relatively small compared to the number of unlabeled nodes, and the time complexity is not burdensome. The second step is to update $\bar{\omega}_k^*$, which is related to the number of latent factors $K$. 
Upon detailed equations in the Appendix, the bottleneck lies in the process of computing the inverse matrix.
The inverse computation can be optimized by using the Optimized CW-like algorithms with a time complexity of $O(K^{2.373})$. 



\begin{table*}
\caption{Link prediction results (\%) on three benchmark datasets. In each column, the \textbf{boldfaced} score denotes the best result, and the \underline{underlined} score represents the second-best result. A.R. refers to the average ranking.}
\centering
\begin{tabular}{l|cc|cc|cc|c}
\hline
\multicolumn{1}{c|}{\multirow{2}{*}{Model}} & \multicolumn{2}{c|}{Cora} & \multicolumn{2}{c|}{Citeseer} & \multicolumn{2}{c|}{Pubmed} & \multicolumn{1}{c}{\multirow{2}{*}{A.R.}} \\ \cline{2-7} 
\multicolumn{1}{c|}{} & AUC & AP & AUC & AP & AUC & AP \\ \hline
GAE & 91.09 ± 0.01 & 92.83 ± 0.03 & 90.52 ± 0.04 & 91.68 ± 0.05 & 96.40 ± 0.01 & 96.50 ± 0.02 & 9.17 \\
VGAE & 91.40 ± 0.01 & 92.60 ± 0.01 & 90.80 ± 0.02 & 92.00 ± 0.02 & 94.40 ± 0.02 & 94.70 ± 0.02 & 10.08 \\
ARGA & 92.40 ± 0.01 & 93.23 ± 0.00 & 91.94 ± 0.00 & 93.03 ± 0.00 & 96.81 ± 0.00 & 97.11 ± 0.00 & 6.25 \\
ARVGA & 92.40 ± 0.01 & 92.60 ± 0.00 & 92.40 ± 0.00 & 93.00 ± 0.00 & 96.50 ± 0.00 & 96.80 ± 0.00 & 7.50 \\ 
DGI & 93.88 ± 1.00 & 93.60 ± 1.14 & 95.98 ± 0.72 & 96.18 ± 0.68 & 96.30 ± 0.20 & 95.65 ± 0.26 & 5.92 \\
MVGRL & 93.33 ± 0.68 & 92.95 ± 0.82 & 88.66 ± 5.27 & 89.37 ± 4.55 & 95.89 ± 0.22 & 95.53 ± 0.30 & 9.50 \\
GRACE & 82.67 ± 0.27 & 82.36 ± 0.24 & 87.74 ± 0.96 & 86.92 ± 1.11 & 94.09 ± 0.92 & 93.26 ± 1.20 & 13.00 \\
CCA-SSG & 93.88 ± 0.95 & 93.74 ± 1.15 & 94.69 ± 0.95 & 95.06 ± 0.91 & 96.63 ± 0.15 & 95.97 ± 0.23 & 5.75 \\ \hline
MaskGAE & \underline{96.55 ± 0.13} & \underline{96.05 ± 0.16} & \underline{97.74 ± 0.14} & \textbf{97.99 ± 0.12} & \textbf{98.74 ± 0.04} & \underline{98.64 ± 0.06} & 1.67 \\ 
GraphMAE2 & 92.32 ± 1.66 & 90.98 ± 2.12 & 87.20 ± 4.30 & 86.29 ± 4.50 & 92.89 ± 0.44 & 92.46 ± 0.38 & 12.83 \\ 
GigaMAE & 94.48 ± 0.12 & 94.09 ± 0.21 & 95.11 ± 0.11 & 95.41 ± 0.11 & 93.56 ± 0.82 & 92.42 ± 0.92 & 7.50 \\ 
AUG-MAE & 90.51 ± 0.57 & 89.82 ± 0.62 & 90.63 ± 0.79 & 91.44 ± 0.90 & 94.69 ± 0.71 & 94.10 ± 0.88 & 11.33 \\
Bandana & 95.83 ± 0.06 & 95.38 ± 0.11 & 96.70 ± 0.31 & 97.04 ± 0.38 & 97.31 ± 0.09 & 96.82 ± 0.24 & 3.17 \\
\hline
\mname{} & \textbf{96.93 ± 0.16} & \textbf{96.76 ± 0.15} & \textbf{97.90 ± 0.52} & \underline{97.96 ± 0.60} & \underline{98.66 ± 0.07} & \textbf{98.65 ± 0.09} & \textbf{1.33} \\ \hline
\end{tabular}
\label{tab:linkpred}
\end{table*}

\section{EXPERIMENTS}
\subsection{Experimental Setting}

We evaluate our methodology on seven widely recognized benchmark datasets. These include four citation benchmark networks: Cora \cite{mccallum2000automating}, Citeseer \cite{giles1998citeseer}, Pubmed \cite{pubmed}, and ogbn-Arxiv \cite{hu2020open}, which have been extensively utilized in seminal works such as \cite{kipf2017, felix2019}. Additionally, to rigorously test our approach, we include two co-purchasing networks: Photo \cite{mcauley2015image} and Computers \cite{mcauley2015inferring}, where contrastive learning baselines have been shown to outperform. Lastly, we evaluated using the WikiCS \cite{mernyei2020wikics} dataset. The result of the WikiCS offers an evaluation setting with a lower homophily ratio. For crafting initial pseudo-labels, R-DGAE \cite{MrabahBTK23} is used as the default clustering algorithm and DinkNet \cite{XZYLL23}for large datasets.
We apply the default hyperparameter settings for graph clustering.
The kernel function for HSIC for the features and pseudo labels are set to Gaussian and Delta, respectively.
We provide the code of \mname{} at the following URL.\footnote{\url{https://github.com/vapene/CIMAGE}}

We benchmark \mname{} against thirteen different models including four generative-based methods, (V)GAE~\cite{Kipf2016} and AR(V)GA~\cite{Pan2018}; four contrastive-based methods, DGI~\cite{velickociv2018}, MVGRL~\cite{Hassani2020}, GRACE~\cite{Zhu2020}, and CCA-SSG~\cite{Zhang2021}; two methods based on masked graph auto-encoder, MaskGAE~\cite{MaskGAE}, GraphMAE2~\cite{Hou2023}, and three recent baselines: GiGaMAE~\cite{ShiDTLL23}, AUG-MAE~\cite{wang2024rethinking}, and Bandana~\cite{bandana}. 
For node classification, we utilize standard splits provided by pytorch geometric or a 1:1:8 (training-validation-testing) split if not provided. For the node classification task, we adopt the linear evaluation protocol widely used in prior studies ~\cite{dgi,MaskGAE} to evaluate the quality of representations learned through the self-supervised learning (SSL) pretext task. In this approach, we freeze the encoder's parameters to generate graph embeddings, then apply a new linear layer to these fixed embeddings for classification.

For the node classification task, we follow the linear evaluation protocol widely adopted in previous studies, such as~\cite{dgi, MaskGAE}. Specifically, the encoder parameters are frozen, and it is used to generate graph embeddings. A new linear layer is then trained on these fixed embeddings to assess the quality of the representations learned through the self-supervised learning (SSL) pretext task.

In the link prediction task, if models were not initially designed for edge reconstruction, namely DGI, MVGRL, GRACE, CCA-SSG, GraphMAE2, and AUG-MAE, we adapted the approach presented in ~\cite{Kipf2016} to generate edge using learned node features. Specifically, a dot product decoder was applied to the learned representations using 85\% of the edges. We randomly sampled an equivalent number of non-existing edges to compute the AUC and AP scores to match the remaining test edges. 

\begin{table}
\caption{Performance comparison with and without latent factor reconstruction, $h_{\theta_{\text{ch}}}$. \textbf{Boldfaced} scores denote the best result.}
\resizebox{\columnwidth}{!}{
\begin{tabular}{l|l|l|l|l}
\hline
Dataset & Metric & Without $h_{\theta_{\text{ch}}}$ & With $h_{\theta_{\text{ch}}}$ & Difference \\ 
\hline
Cora & Link AUC & 96.16 ± 0.65 & \textbf{96.93 ± 0.16} & + 0.77 \\
Cora & Link AP & 96.39 ± 0.63 & \textbf{96.76 ± 0.15} & + 0.37 \\
Cora & Node ACC & 83.90 ± 0.98 & \textbf{84.40 ± 0.72} & + 0.50 \\
\midrule
Citeseer & Link AUC & 97.64 ± 0.64 & \textbf{97.90 ± 0.52} & + 0.26 \\
Citeseer & Link AP & 97.52 ± 0.73 & \textbf{97.96 ± 0.60} & + 0.44 \\
Citeseer & Node ACC & 73.42 ± 0.98 & \textbf{74.45 ± 0.32} & + 1.03 \\
\hline
\end{tabular}
}
\label{ablation:channeldecoder}
\end{table}

\subsection{Performance on Node Classification}

The performance results for node classification are presented in Table \ref{tab:noderesult}, where \mname{} demonstrates superior performance compared with established baselines. This distinction is particularly evident in Citeseer and Arxiv datasets, which significantly underscore the robustness and effectiveness of our proposed framework in handling node classification tasks. A point of particular interest within these results is the relatively lower performance of the Graph Auto-Encoder (GAE), which serves as a foundational pretext task in our study. 
This performance difference between GAE and \mname{} demonstrates the superiority of mRMR-based self-supervised learning by explicitly incorporating conditional independence.
\mname{} exhibits superior overall performance relative to the others as reflected in the average ranking metric (A.R) within Table~\ref{tab:noderesult}.

\subsection{Performance on Link Prediction}

As shown in Table \ref{tab:linkpred}, \mname{} demonstrates superior performance compared to most baselines, except MaskGAE. We attribute this performance to the edge reconstruction scheme of the pretext task. This method is also employed by MaskGAE, which accounts for the comparable performance. However, \mname{} achieves a marginally higher average rank. This incremental enhancement in performance highlights the benefits of incorporating conditional independence-aware masking, suggesting a promising avenue for advancements in link prediction within the self-supervised learning domain.

\subsection{Ablation Studies}
In this section, we conduct five ablation studies on the main components of \mname{} to understand their roles comprehensively. 
\subsubsection{Effect of latent factor reconstruction $h_{\theta_{\text{ch}}}$.}
Considering the importance of CI-aware masking, we investigate whether the proposed context reconstruction strategy enhances performance. Table \ref{ablation:channeldecoder} shows the performance of both link prediction and node classification with and without the latent factor reconstruction on two different datasets. We can observe that conditional independence subset reconstruction enhances both link prediction and node classification by a notable margin, indicating the contribution to learning the rich graph representation.


\begin{table}
    \caption{Performance comparison on node classification and link prediction tasks for Cora and Citeseer datasets with different values of $K$. In each column, the \textbf{boldfaced} score denotes the best result.}
    \vspace{-3mm}
    \centering
    \resizebox{\linewidth}{!}{ 
    \begin{tabular}{l|l|ccccc}
        \toprule
        Dataset & Metric & $K=4$ & $K=8$ & $K=16$ & $K=32$ & $K=64$ \\
        \midrule
        \multirow{3}{*}{Cora} & Link AUC & 95.69 ± 0.74 & 96.01 ± 0.94 & \textbf{96.93 ± 0.16} & 94.76 ± 0.75 & 94.74 ± 1.04 \\
        & Link AP & 95.76 ± 0.57 & 96.15 ± 0.77 & \textbf{96.76 ± 0.15} & 95.27 ± 0.67 & 95.26 ± 1.08 \\
        & Node ACC & 82.35 ± 0.72 & 83.45 ± 0.68 & \textbf{84.40 ± 0.72} & 80.09±0.78 & 79.74 ± 0.87 \\
        \midrule
        \multirow{3}{*}{Citeseer} & Link AUC & 93.30 ± 1.26 & 95.64 ± 0.76 & 96.89 ± 0.68 & \textbf{97.90 ± 0.52} & 96.36 ± 0.66 \\
        & Link AP & 92.90 ± 1.65 & 95.64 ± 0.88 & 96.99 ± 0.81 & \textbf{97.96 ± 0.60} & 96.76 ± 0.71 \\
        & Node ACC & 72.17 ± 0.72 & 72.28 ± 0.89 & \textbf{74.45 ± 0.32} & 73.97 ± 0.89 & 69.73 ± 1.16 \\
        \bottomrule
    \end{tabular}
    }
    \vspace{-2mm}
    \label{ablation:numchannels}
\end{table}

\begin{table}
    \caption{Effect of single view embeddings on node classification and the number of factors for each view}
    \vspace{-2mm}
    \centering
    \small 
    \begin{tabular}{l|l|ccc}
        \toprule
         Dataset & Metric & $\bm{\mathcal{\mathcal{F}}}_1$ & $\bm{\mathcal{\mathcal{F}}}_2$ & $[\bm{\mathcal{F}_1}\! \mid \! \bm{\mathcal{F}_2}]$\\ 
        \midrule
        \multirow{2}{*}{Cora} & Node ACC & 82.90 ± 0.89 & 83.10 ± 0.77 & 84.40 ± 0.72\\ 
        & num. factors & 8 & 8 & 16 \\
        \midrule
        \multirow{2}{*}{Citeseer} & Node ACC & 71.99 ± 0.78 & 72.32 ± 0.24 & 74.45 ± 0.32 \\ 
        & num. factors & 5 & 11 & 16 \\
        \bottomrule
    \end{tabular}
    \vspace{-1mm}
    \label{tab:minredun}
\end{table}

\subsubsection{Assessing the minimum redundancy-aware masking method}
Eq.~\eqref{eq:F1=F2} in the theoretical analysis section indicates that $\bm{\mathcal{F}}_1$ or $\bm{\mathcal{F}}_2$ individually should yield accuracy equivalent to predictions using both $\bm{\mathcal{F}}_1$ and $\bm{\mathcal{F}}_2$ concurrently. Table \ref{tab:minredun} presents the node classification results when only a single view is employed. 
Even though the performance marginally decreases because of the shrinkage of feature dimensions, relatively consistent performance demonstrates that both $\bm{\mathcal{F}}_1$ and $\bm{\mathcal{F}}_2$ satisfy minimum redundancy.
$ \bm{\mathcal{F}}_1$ in the Citeseer dataset shows a higher decrement because fewer factors are assigned.

\subsubsection{Effect of latent factor dimension size $D_{ch}$.}
Table \ref{ablation:performance_metrics_D_{ch}} ablates the effect of varying embedding sizes of latent factors. Embedding size reflects the effectiveness of information compression. 
Although our input for the structure reconstruction function $h_{{\theta}_{\text{st}}}$ has an overall feature dimension ($D_{ch} * K$) of 512 or 1024, we only leverage comparably small $D_{ch}$ for each factor for distinguishing the subsets that satisfy conditional independence. Therefore, \mname{} has to find a balance that is not too small for $D_{ch}$ but not too oversized $D_{ch} * K$. The result of different $D_{ch}$ is presented in Table \ref{ablation:performance_metrics_D_{ch}} for two datasets. Notably, a dimension size of 32 yields the highest link prediction and node classification scores.

\begin{table}
\caption{Performance comparison on node classification and link prediction tasks for Cora and Citeseer datasets with different values of latent factor dimension $D_{ch}$. \textbf{Boldfaced} scores denote the best results.}
\resizebox{1.0\columnwidth}{!}{
\begin{tabular}{l|l|cccc}
    \toprule
    Dataset & Metric & $D_{ch}=8$ & $D_{ch}=16$ & $D_{ch}=32$ & $D_{ch}=64$ \\
    \midrule
    \multirow{3}{*}{Cora} & Link AUC  & 95.88 ± 0.64 & 96.27 ± 0.73 & \textbf{96.93 ± 0.16} & 94.21 ± 0.95 \\
    & Link AP & 96.01 ± 0.59 & 96.51 ± 0.56 & \textbf{96.76 ± 0.15} & 94.92 ± 0.78 \\
    & Node ACC & 83.46 ± 0.81 & 83.75 ± 0.86 & \textbf{84.40 ± 0.72} & 80.91 ± 0.81 \\
    \midrule
    \multirow{3}{*}{Citeseer} & Link AUC & 97.15 ± 0.55 & 97.30 ± 0.59 & \textbf{97.90 ± 0.52} & 96.12 ± 0.78 \\
    & Link AP & 97.32 ± 0.63 & 97.52 ± 0.61 & \textbf{97.96 ± 0.60} & 96.52 ± 0.79 \\
    & Node ACC & 72.95 ± 1.24 & 73.02 ± 1.07 & \textbf{74.45 ± 0.32} & 72.00 ± 0.97 \\
    \bottomrule
\end{tabular}
}
\label{ablation:performance_metrics_D_{ch}}
\end{table}

\subsubsection{Assessing the Impact of Clustering Performance.}


Figure 2 demonstrates how pseudo-labels' quality affects the model's capacity to learn expressive representations. We fixed the number of labeled nodes identically for a fair comparison. At a clustering accuracy of 20\%, the node classification accuracy stands at merely 71.15\%. This accuracy improves as the quality of pseudo-labels increases, with a 45\% quality pseudo-label achieving a 72.53\%. Remarkably, with a clustering accuracy of 68\%, which matches the performance of a graph clustering of \mname{}, the method achieves state-of-the-art accuracy in downstream tasks. Further increasing the accuracy of pseudo-labels to 100\% results in an impressive accuracy of 75.12\%. 
This experiment validates our approach's effectiveness and suggests the potential for even greater improvements with advancements in clustering algorithms.

\subsubsection{Ablation study on the effect of factor number $K$}
The number of factor $K$ determines the subset to satisfy conditional independence. A smaller value of $K$ often leads to challenges in achieving conditional independence among the decomposed subsets because of the overlap and redundancy among factors. Conversely, an excessively large $K$ may lead to the generation of factors that are either trivial or lack substantive relevance. Therefore, we conducted an ablation study about the number of factors. 
From Table~\ref{ablation:numchannels}, a value of 16 or 32 is appropriate for $K$. The results show identical patterns for both Cora and Citeseer datasets. It increases slowly until the best number, then decreases rapidly by around 4\%. In Citeseer, optimal $K$ is different between node classification and link prediction. 

\begin{figure}[t]
\centering
\includegraphics[width=0.9\columnwidth, angle=0]{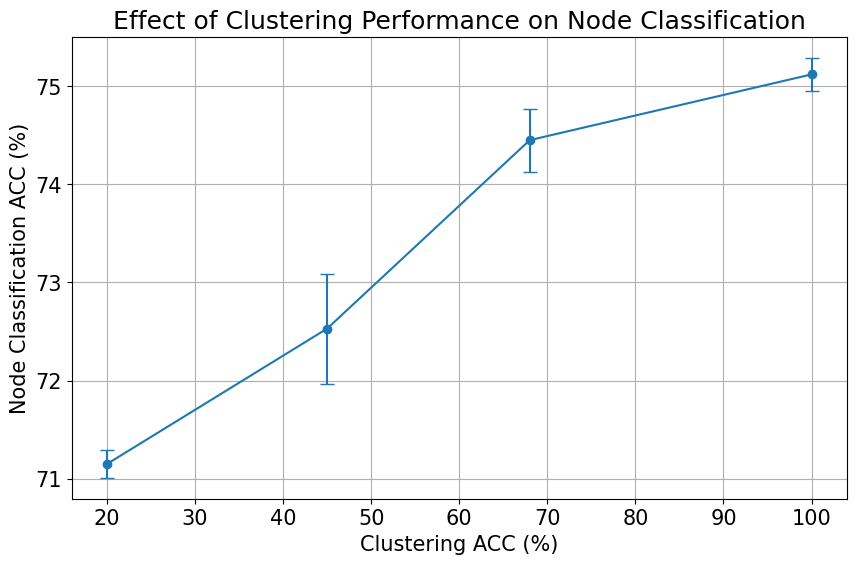}
\vspace{-2mm}
\caption{Performance variation in node classification relative to clustering accuracy on the Citeseer dataset, with standard deviation represented by error bars.}
\label{fig:taskagnostic}
\vspace{-2mm}
\end{figure}

\section{CONCLUSION}
We introduced \mname{}, a novel masked graph auto-encoder framework that explicitly leverages conditional independence to enhance representation quality. Our approach goes beyond traditional masked graph auto-encoders by decomposing latent features into two subsets that satisfy conditional independence with pseudo-labels obtained through unsupervised graph clustering, thereby achieving minimum redundancy and maximum relevance. \mname{} also introduces a novel pretext task that reconstructs these decomposed subsets to further refine the learned representations. Empirical results underscore the effectiveness of \mname{}, demonstrating substantial improvements in both node classification and link prediction tasks. 

\section*{Acknowledgement}
This work was supported by the Institute of Information \& communications Technology Planning \& evaluation (IITP) grant and National Research Foundation of Korea (NRF) grant funded by the Korea government (MSIT) (RS-2019-II190421, RS-2020-II201821, RS-2024-00438686, RS-2024-00436936, RS-2024-00360227, RS-2023-0022544, NRF-2021M3H4A1A02056037, RS-2024-00448809).  This research was also partially supported by the Culture, Sports, and Tourism R\&D Program through the Korea Creative Content Agency grant funded by the Ministry of Culture, Sports and Tourism in 2024 (RS-2024-00333068, RS-2024-00348469 (25\%)).


\clearpage
\bibliographystyle{ACM-Reference-Format}
\balance
\bibliography{WSDM}

\clearpage
\appendix
\section{Appendix}

\newcolumntype{R}{>{\raggedleft\arraybackslash}X}
\begin{table}[t]
\caption{Link prediction hyperparameters}
\centering
\begin{tabularx}{\columnwidth}{lRRR}
\toprule
Link prediction & Cora & Citeseer & Pubmed \\
\midrule
\# of factors \( K \) & 16 & 32 & 32 \\
factor dim. & 32 & 32 & 32 \\
joint learning reg. \( \lambda_1 \) & 0.138 & 0.61 & 0.25 \\
masking rate & 0.7 & 0.7 & 0.7 \\
encoder hidden dim. & 512 & 512 & 512 \\
structure decoder hidden dim. & 32 & 64 & 32 \\
factor recon. hidden dim. & 256 & 256 & 256 \\
clustering reg. \( \lambda_2 \) & 0.1     & 0.2 & 0.2 \\
\bottomrule
\end{tabularx}
\label{linkprediction}
\end{table}

\begin{table*}[t]
    \caption{Node classification hyperparameters}
    \centering
    \resizebox{0.9\linewidth}{!}{ 
    \begin{tabular}{l*{7}{r}}
        \toprule
        Node classification & Cora & Citeseer & Pubmed & Photo & Computers & WikiCS & Arxiv \\
        \midrule
        \# of factors \( K \) & 16 & 16 & 16 & 32 & 32 & 32 & 16 \\
        factor dim. & 32 & 32 & 32 & 16 & 16 & 16 & 32 \\
        joint learning reg. \( \lambda_1 \) & 0.86 & 0.77 & 0.63 & 0.70 & 0.32 & 0.60 & 0.82 \\
        masking rate & 0.7 & 0.7 & 0.7 & 0.7 & 0.7 & 0.7 & 0.7 \\
        encoder hidden dim. & 512 & 512 & 512 & 512 & 512 & 512 & 256 \\
        structure recon. hidden dim. & 32 & 32 & 32 & 32 & 64 & 32 & 128 \\
        factor recon. hidden dim. & 256 & 256 & 256 & 256 & 256 & 256 & 128 \\
        clustering reg. \( \lambda_2 \) & 0.4 & 0.4 & 0.2 & 0.1 & 0.2 & 0.1 & 0.1 \\
        \bottomrule
    \end{tabular}
    }
    \label{nodeclassificationhyp}
\end{table*}

\subsection{Data statistics}
Table \ref{datastatistics} presents the data statistics utilized for our experiments. It highlights variations in densities and homophily ratios alongside other fundamental properties.

\newcolumntype{R}{>{\raggedleft\arraybackslash}X}
\begin{table*}[t]
\centering
\begin{tabularx}{\textwidth}{lRRRRRRR}
\toprule
 &\# of nodes &\# of edges &Dim. of features &\# of classes &Edge density &Homophily ratio \\ 
\midrule
Cora & 2,708 & 10,556 & 1,433 & 7 & $2.9e^{-3}$ & 0.81 \\ 
Citeseer & 3,327 & 9,104 & 3,703 & 6 & $1.6e^{-3}$ & 0.74 \\ 
Photo & 7,650 & 238,162 & 745 & 8 & $8.1e^{-3}$ & 0.83 \\ 
Computers & 13,752 & 491,722 & 767 & 10 & $5.2e^{-3}$ & 0.78 \\ 
Pubmed & 19,717 & 88,648 & 500 & 3 & $4.6e^{-4}$ & 0.80 \\ 
WikiCS & 11,701 & 216,123 & 300 & 10 & $6.3e^{-3}$ & 0.65 \\ 
Arxiv & 169,343 & 1,166,243 & 128 & 40 & $8.1e^{-5}$& 0.66 \\
\bottomrule
\end{tabularx}
\caption{Dataset statistics}
\label{datastatistics}
\end{table*}

\subsubsection{Hyper-parameters}
Table \ref{linkprediction} provides hyperparameters for the link prediction task, and Table \ref{nodeclassificationhyp} provides hyperparameters for the node classification task.

\subsection{Downstream sample complexity}
Theorem~\ref{theorem:optimaldecoder} has shown that utilizing the optimal pretext task lowers the number of labeled data required for the downstream node classification task between $\tilde{O}(\lvert \hat{\bm{C}}\rvert)$ and $\tilde{O}(\lvert D_{ch} \rvert)$. Figure \ref{fig:lowsamplecomplexity} illustrates the variation in node classification accuracy across different labeled data ratios. \mname{} remains consistent up to a 5\% training ratio, staying within one standard deviation. The gap between \mname{} and GCN expands from 2.42\% to 7.71\% as the training ratio reduces to 1\%, signifying \mname{}'s enhanced robustness to smaller training ratios. 

However, there is a significant decrease in performance when moving from a 5\% to a 1\% training ratio. This deviation from theoretical expectations could be considered an incomplete CI error. \mname{} decomposes $\bm{\mathcal{F}}_1$ and $\bm{\mathcal{F}}_2$, but such an optimization process requires a sufficient number of labeled data. Although using a 5\% train ratio is acceptable since it is already less than the standard setting provided in Pytorch geometric, developing a more efficient CI measure module would benefit future work. For these evaluations, the validation ratio is maintained at 0.2\%, with the remainder split into training and testing data. The standard deviation is computed over three trials.

\begin{figure}[t]
\centering
\includegraphics[width=1.0\columnwidth, angle=0]{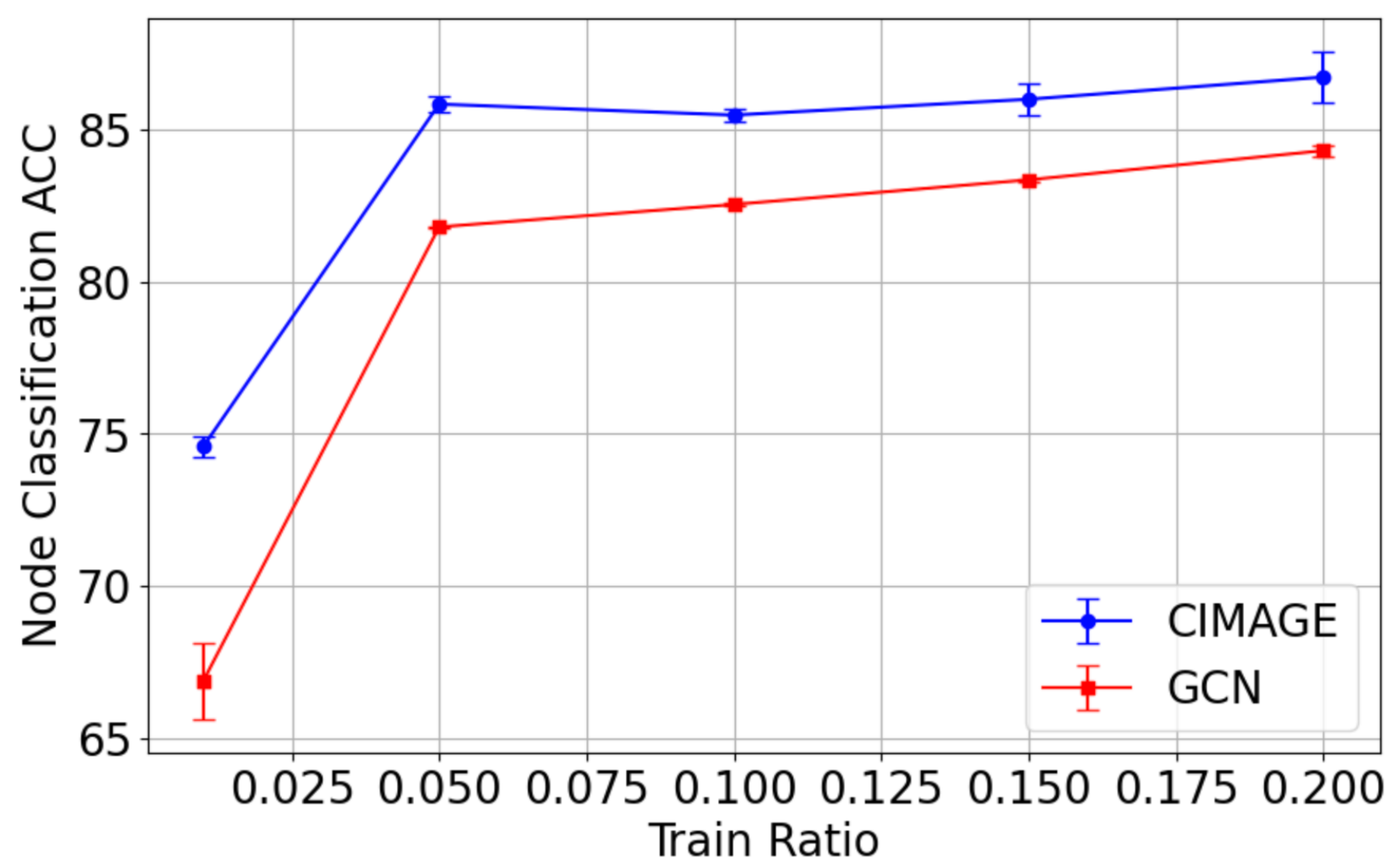}
\caption{Robustness to low downstream sample complexity.}
\label{fig:lowsamplecomplexity}
\end{figure}

\subsection{Derivation for Closed-Form Equation (\ref{eq:CMIHSIC})}
The derivation follows the one presented in ~\cite{Koyama2022}. 

CMI-HSIC is represented with the equation below:
\[
\argmax_{\bm{\omega}} \sum_{i=1}^{D}{\omega_i \widehat{\operatorname{HSIC}}(\bm{f}_i,\hat{\bm{C}})} - (1-\pi)\sum_{i,j=1}^{D}{\omega_i \omega_j \widehat{\operatorname{HSIC}}(\bm{f}_i,\bm{f}_j)} - \beta\lvert\bm{\omega}\rvert_1.
\]

Suppose \textbf{K}, \textbf{L} are the corresponding kernel matrices of $\bm{f}_i$ and $\hat{\bm{C}}$, respectively. We used a centered Gram matrix $\mathbf{\bar{K}} = (\bm{I}-\frac{1}{n}\mathbf{1}\mathbf{1}^\mathsf{T})\textbf{K}(\bm{I}-\frac{1}{n}\mathbf{1}\mathbf{1}^\mathsf{T})$ and $\mathbf{\bar{L}}$ to compute $\widehat{\operatorname{HSIC}}(\bm{f}_{i},\hat{\bm{C}}) = \mathrm{Tr}(\bar{\textbf{K}}^{(i)}\bar{\textbf{L}})$. Assuming $\pi$ as a constant and $\bar{\textbf{L}} \sim \mathcal{N}(\bar{\textbf{K}}\bm{\omega}, \sigma^2\mathbf{I})$, where $\sigma$ is the standard deviation of i.i.d. noise in the classical linear model. Now assume $\pi=\frac{1}{2}$, CMI-HSIC can be written as below:
\begin{equation}
\begin{aligned}
    &\argmin_{\bm{\omega}} \Bigg[ -\sum^{D}_{i=1} \omega_i \widehat{\operatorname{HSIC}}(\bm{f}_i,\hat{\bm{C}}) + \frac{1}{2} \sum_{i,j=1}^{D} \omega_i \omega_j \widehat{\operatorname{HSIC}}(\bm{f}_i,\bm{f}_j) \\
    &\qquad + \frac{1}{2} \widehat{\operatorname{HSIC}}(\hat{\bm{C}},\hat{\bm{C}}) \Bigg].
    \end{aligned}
\end{equation}
The objective could be equivalently written as:
\begin{equation}
\begin{aligned}
    &\frac{1}{2}\Bigg[ -\frac{1}{2}\sum^{D}_{i=1} \omega_i \bar{\textbf{K}}^{(i)}\bar{\textbf{L}} + \sum_{i,j=1}^{D} \omega_i \omega_j \bar{\textbf{K}}^{(i)}\bar{\textbf{K}}^{(j)} + \bar{\textbf{L}}\bar{\textbf{L}} \Bigg] \\
    &\approx \lVert \bar{\textbf{L}} - \sum^{D}_{i=1} \omega_i \bar{\textbf{K}}^{(i)} \rVert^2.
\end{aligned}
\end{equation}
The above equation has been converted to a linear lasso objective of \cite{Tibshirani1996}. We set the prior \( p(\bm{\omega} \mid \bm{\nu}) \) as the Student's t distribution in the form \( \prod_{i} (1+\frac{\omega_i^2}{\nu_i})^{-\frac{\nu_i+1}{2}} \), such that the vector \( \bm{\nu} >0 \). This aligns with the standard assumptions of the Bayesian linear model, which is \( p(\bm{\omega}) = \prod_{i} p(\omega_i) \) and that \( p(\bm{\omega}) \) is symmetric. Then, the non-Gaussian prior's density \( p(\bm{\omega}) \) can be represented as a supremum over Gaussian functions and can be represented by some function \( \varphi(\bm{\xi}) \).
\begin{equation}
\begin{aligned}
&\log p(\bm{\omega}) = \sup_{\bm{\xi}>0}\log N(\bm{\omega};0,\bm{\xi}^{-1})\varphi(\bm{\xi}) \\
&= \sup_{\bm{\xi}>0}(-\frac{\bm{\omega}^2}{2\bm{\xi}}-\varphi(\bm{\xi})). 
\end{aligned}
\label{eq:nongaussian}
\end{equation}
Building upon this transformation by ~\cite{palmer2005}, the $\bm{\omega}$ distribution would be represented below.
\begin{align}
p(\omega_i|\nu_i) &= q(\lvert \omega_i \rvert \nu_{i}^{\frac{1}{2}})\nu_{i}^{\frac{1}{2}} \notag \\
&= \nu_{i}^{\frac{1}{2}} \operatorname{exp}\left(\log q(\lvert \omega_i \rvert \nu_{i}^{\frac{1}{2}})\right).
\end{align}
By using Equation (\ref{eq:nongaussian}),
\begin{align}
p(\omega_i|\nu_i) &= q(\lvert \omega_i \rvert \nu_{i}^{\frac{1}{2}})\nu_{i}^{\frac{1}{2}} \notag \\
&= \nu_{i}^{\frac{1}{2}} \sup_{\xi>0}\operatorname{exp}(-\frac{\omega_i^2 \nu_i}{2\xi_i}-\varphi(\xi_i)).
\end{align}
By representing as a Gaussian variable,
\begin{align}
p(\omega_i|\nu_i) &= q(\lvert \omega_i \rvert \nu_{i}^{\frac{1}{2}})\nu_{i}^{\frac{1}{2}} \notag \\
&= \nu_{i}^{\frac{1}{2}} \sup_{\xi>0} N(\omega_i \mid 0, \frac{\xi_i}{\nu_i})(\frac{2\pi\xi_i}{\nu_i})^{\frac{1}{2}}\operatorname{exp}(-\varphi(\xi_i)).
\label{eq:transformation33}
\end{align}

With an assumption $\bar{\textbf{L}} \sim N(\bar{\textbf{K}}\bm{\omega}, \sigma^2\mathbf{I})$, together with Equation (\ref{eq:transformation33}), the variational bound for $\log p(\bar{\hat{\bm{C}}}\mid \bm{\nu})$ can be defined. 
\begin{equation}
\begin{aligned}
&\log \int_{\mathbb{R}^D} N(\bar{\textbf{L}} \mid \bar{\textbf{K}}\bm{\omega}, \sigma^2\mathbf{I}) p(\bm{\omega} \mid \bm{\nu}) \\
&\quad = \log \int_{\mathbb{R}^D} N_1 \prod^{D}\nu_{i}^{\frac{1}{2}} \sup_{\xi>0} N_2\left(\frac{2\pi\xi_i}{\nu_i}\right)^{\frac{1}{2}}\operatorname{exp}(-\varphi(\xi_i)).
\end{aligned}
\end{equation}
Given that the logarithm is monotonically increasing,
\begin{equation}
\begin{aligned}
&\text{Lower bound is defined below.} \\
&\quad \sup_{\xi>0} \left[ \log \int_{\mathbb{R}^D} N_3 d\omega \right. + \left. \log \prod^{D}L_i(\xi_i) \right] \\
&\text{where } N_1 = N(\bar{\textbf{L}} \mid \bar{\textbf{K}}\bm{\omega}, \sigma^2\mathbf{I}), \\
&N_2 = N(\omega_i \mid 0, \frac{\xi_i}{\nu_i}), \\
&N_3 = N(\bar{\textbf{L}} \mid \bar{\textbf{K}}\bm{\omega}, \sigma^2\mathbf{I}) N(\bm{\omega} \mid \bm{0}, \bm{\Theta}^{-1}), \\
&L_i(\xi_i) = \nu_{i}^{\frac{1}{2}} \left(\frac{2\pi\xi_i}{\nu_i}\right)^{\frac{1}{2}}\operatorname{exp}(-\varphi(\xi_i)), \Theta = \prod^{P}\frac{\bm{\nu}}{\bm{\xi}}.
\end{aligned}
\end{equation}
Equation (\ref{eq:finalVB}) gives the final variational bound for $\log p(\bar{\textbf{L}}\mid \bm{\nu})$.
\begin{align}
&\log p(\bar{\textbf{L}}\mid \bm{\nu}) \approx -\inf_{\bm{\mu},\bm{\Sigma},\bm{\xi}} \frac{1}{2}\bigg[\frac{1}{\sigma}\lVert \bar{\textbf{L}} - \bar{\textbf{K}}\bm{\mu}\rVert^2 + \bm{\mu}^{\mathsf{T}}\bm{\Theta}\bm{\mu} \nonumber \\
&\quad + \frac{1}{\sigma}\mathrm{Tr}(\bar{\textbf{K}}^{\mathsf{T}}\bar{\textbf{K}}\bm{\Sigma}) + \mathrm{Tr}(\bm{\Theta}\bm{\Sigma}) - \log \det \bm{\Sigma} \nonumber \\
&\quad + \sum^{P}(2\varphi(\xi_p)-\log \nu_p) + N^2\log(2\sigma^2) \bigg].
\label{eq:finalVB}
\end{align}
With posterior $N(\bm{\omega} \mid \bm{\mu}, \bm{\Sigma})$, likelihood $N(\bar{\textbf{L}}\mid \bar{\textbf{K}}\bm{\omega}, \sigma^2\mathbf{I})$, and prior $N(\bm{\omega}\mid \bm{0},\bm{\Theta}^{-1})$. Therefore, the closed-form update equation for $\bm{\mu}$ can be defined as below by only considering related terms of $\bm{\mu}$ in Equation (\ref{eq:closedform}).
\begin{equation}
\label{eq:closedform}
    \bm{\mu} = \arginf_{\bm{\mu}} \frac{1}{2\sigma^2}\lVert \bar{\textbf{L}}-\bar{\textbf{K}}\bm{\mu} \rVert^2_2 + \frac{1}{2}\bm{\mu}^{\mathsf{T}}\bm{\Theta}\bm{\mu}+\beta \lVert\bm{\mu}\rVert_1
\end{equation}

The rest of the updated equations are presented below with element-wise product $\odot$ and element-wise division $\oslash$. 
\begin{align}
    \bm{\Sigma} &= \frac{\partial}{\partial \bm{\Sigma}}\left[\frac{1}{\sigma^2}\mathrm{Tr}(\bar{\textbf{K}}^{\mathsf{T}}\bar{\textbf{K}}\bm{\Sigma})+\frac{1}{2}\mathrm{Tr}(\bm{\Theta}\bm{\Sigma})-\frac{1}{2}\log\det\bm{\Sigma}\right] \notag \\ 
    &= \sigma^2(\bar{\textbf{K}}^{\mathsf{T}}\bar{\textbf{K}}+\sigma^2\bm{\Theta})^{-1}.
    \label{eq:21}
\end{align}
\begin{equation}
    \bm{\xi} = \arginf_{\bm{\xi}}  \frac{1}{2}\bm{\mu}^{\mathsf{T}}\bm{\Theta}\bm{\mu} + \frac{1}{2}\mathrm{Tr}(\bm{\Theta}\bm{\Sigma}) +\sum^{P}\varphi(\xi_p)
\end{equation}
\begin{equation}
    \bm{\nu} = \frac{\partial}{\partial\bm{\nu}}[\bm{\mu}^{\mathsf{T}}\bm{\Theta}\bm{\mu}+\mathrm{Tr}(\bm{\Theta}\bm{\Sigma})-\sum^{P}\log\nu_i] =  \bm{\xi} \oslash (\bm{\mu}\odot \bm{\mu}+\operatorname{diag}\bm{\Sigma}),
\end{equation}
\begin{align}
    \sigma^2 &= \frac{\partial}{\partial\sigma^2}  \left[ \frac{1}{\sigma}\lVert \bar{\textbf{L}} - \bar{\textbf{K}}\bm{\mu}\rVert^2 +\frac{1}{\sigma^2}\mathrm{Tr}(\bar{\textbf{K}}^{\mathsf{T}}\bar{\textbf{K}}\bm{\Sigma}) + N^2\log(\sigma^2) \right] \notag \\
    &= \frac{1}{N^2}(\lVert \bar{\textbf{L}}-\bar{\textbf{K}}\bm{\mu} \rVert^2 + \mathrm{Tr}(\bar{\textbf{K}}^{\mathsf{T}}\bar{\textbf{K}}\bm{\Sigma}))
\end{align}

\subsection{Proof of theorem~\ref{theorem:optimaldecoder}}
Here is a restatement of Theorem~\ref{theorem:optimaldecoder}: 
\theorema*

\begin{proof}[Proof of Theorem~\ref{theorem:optimaldecoder}]
As stated in the derivation of Equation (\ref{eq:closedform}), the prior is expressed as a form of Gaussian, \(N(\bm{\omega}\mid \bm{0}, \bm{\Theta}^{-1})\). This makes \(\bm{\mathcal{F}}_1\), \(\bm{\mathcal{F}}_2\), and \(\hat{\bm{C}}\) jointly Gaussian, and the partial covariance between them can be expressed as \(\bm{\Sigma}_{\bm{\mathcal{F}}_1\bm{\mathcal{F}}_2\mid\hat{\bm{C}}} = \operatorname{Cov}[\bm{\mathcal{F}}_1-\mathbb{E}[\bm{\mathcal{F}}_1\mid\hat{\bm{C}}], \bm{\mathcal{F}}_2-\mathbb{E}[\bm{\mathcal{F}}_2\mid\hat{\bm{C}}]] = \bm{\Sigma}_{\bm{\mathcal{F}}_1\bm{\mathcal{F}}_2}-\bm{\Sigma}_{\bm{\mathcal{F}}_1\hat{\bm{C}}}\bm{\Sigma}_{\hat{\bm{C}}\hat{\bm{C}}}^{-1}\bm{\Sigma}_{\bm{\mathcal{F}}_2\hat{\bm{C}}}\). 

Following the proof by \cite{Lee2021}, we show the optimal reconstruction $\hat{\bm{\mathcal{F}}_2}^*$ can separate downstream label without any approximation. Consider our factor reconstruction function.
\begin{equation}
    h_{\theta_{\text{ch}}}(\bm{\mathcal{F}}_1) = \mathbb{E}[\bm{\mathcal{F}}_2\mid\bm{\mathcal{F}}_1=f_1]
\end{equation}

Together with the assumption that $\operatorname{rank}(\mathbb{E}[\bm{\mathcal{F}}_2 \mid \hat{\bm{C}}=c]) = \lvert \bm{C} \rvert$, the matrix $\bm{\Sigma}_{\bm{\mathcal{F}}_2\hat{\bm{C}}}$ will also have a rank of $\lvert \bm{C} \rvert$. Thus, it has a left inverse, and $\mathbb{E}[\hat{\bm{C}} \mid \bm{\mathcal{F}}_1]$ will be given by $\bm{\Sigma}_{\bm{\mathcal{F}}_2\hat{\bm{C}}}^\dagger \bm{\Sigma}_{\hat{\bm{C}}\hat{\bm{C}}} h_{\theta_{\text{ch}}}(\bm{\mathcal{F}}_1)$ without any approximation. Considering $\bm{\Sigma}_{\bm{\mathcal{F}}_2\hat{\bm{C}}}^\dagger \bm{\Sigma}_{\hat{\bm{C}}\hat{\bm{C}}}$ as an arbitrary matrix $\bm{M}$, then $\mathbb{E}[\hat{\bm{C}} \mid \bm{\mathcal{F}}_1] = \bm{M} h_{\theta_{\text{ch}}}(\bm{\mathcal{F}}_1)$. This expression is identical to a linear transformation of the reconstruction. The operation results in a linear combination of the columns of $h_{\theta_{\text{ch}}}(\bm{\mathcal{F}}_1)$, where the coefficients of the combination are given by the corresponding rows of $\bm{M}$. This operation preserves the properties of vector addition and scalar multiplication.

To show the downstream label complexity, consider the inequality below and assume it naturally holds. 
\begin{equation}
\frac{1}{2n_2}\lVert \hat{\bm{C}} - h_{\theta_{\text{ch}}}(\bm{\mathcal{F}}_1)\rVert^2_F \leq \frac{1}{2n_2}\lVert \hat{\bm{C}}-h_{\theta_{\text{ch}}}^*(\bm{\mathcal{F}}_1)\rVert^2_F
\label{eq:d_1}
\end{equation}
Let \(\bm{\kappa}\) denote the difference \(\hat{\bm{C}}-h_{\theta_{\text{ch}}}^*(\bm{\mathcal{F}}_1)\). Equation (\ref{eq:d_1}) can then be equivalently written as:
\begin{align}
&\lVert h_{\theta_{\text{ch}}}^*(\bm{\mathcal{F}}_1)-h_{\theta_{\text{ch}}}(\bm{\mathcal{F}}_1)\rVert^2_F \leq 2\left\langle\bm{\kappa}, h_{\theta_{\text{ch}}}^*(\bm{\mathcal{F}}_1) - h_{\theta_{\text{ch}}}(\bm{\mathcal{F}}_1)\right\rangle \nonumber \notag \\
&\leq 2\lVert P_{h_{\theta_{\text{ch}}}(\bm{\mathcal{F}}_1)}\bm{\kappa}\rVert_F \lVert h_{\theta_{\text{ch}}}^*(\bm{\mathcal{F}}_1)- h_{\theta_{\text{ch}}}(\bm{\mathcal{F}}_1)\rVert_F. \notag \\
&\text{Therefore,}\notag \\
&\lVert h_{\theta_{\text{ch}}}^*(\bm{\mathcal{F}}_1)-h_{\theta_{\text{ch}}}(\bm{\mathcal{F}}_1)\rVert_F \leq 2\lVert P_{h_{\theta_{\text{ch}}}(\bm{\mathcal{F}}_1)}\bm{\kappa}\rVert_F
\label{eq:d_2}
\end{align}

Suppose column of $\bm{\kappa}$ is i.i.d sampled from $N(0,\bm{\Sigma}_{rr})$ with probability at least $1-\delta$,  
\begin{align}
&\lVert P_{h_{\theta_{\text{ch}}}(\bm{\mathcal{F}}_1)}\bm{\kappa}\rVert^2_F \lesssim \mathrm{Tr}(\bm{\Sigma}_{\bm{\kappa}}) + \lvert \hat{\bm{C}} \rvert +\log(\frac{\dim(\bm{\kappa})}{\delta}) 
\label{eq:d_3}
\end{align}
By combining the results of Equation (\ref{eq:d_2}) and Equation (\ref{eq:d_3}), 
\begin{align}
&\lVert h_{\theta_{\text{ch}}}^*(\bm{\mathcal{F}}_1)-h_{\theta_{\text{ch}}}(\bm{\mathcal{F}}_1)\rVert_F \lesssim \notag \\
&\sqrt{\mathrm{Tr}(\bm{\Sigma}_{\hat{\bm{C}}\hat{\bm{C}}\mid \bm{\mathcal{F}}_1})(\lvert \hat{\bm{C}} \rvert +\log(\frac{\dim(\bm{\kappa})}{\delta}))}
\end{align}
Further assume $n_2 >> \rho^4(\lvert \hat{\bm{C}} \rvert)+\log(\frac{\dim(\bm{\kappa})}{\delta})$, with \\
$\mathbb{E}[h_{\theta_{\text{ch}}}(\bm{\mathcal{F}}_1)h_{\theta_{\text{ch}}}(\bm{\mathcal{F}}_1)^{T}]$ being a $\rho^2-subgaussian$ random variable, will give the approximation bound of $\tilde{O}(\lvert \hat{\bm{C}} \rvert)$, where $\tilde{O}$ notation is used to hide log factors. 
\end{proof}

\end{document}